\documentclass[letterpaper]{article} 
\usepackage{aaai2027}
\usepackage[hyphens]{url}
\usepackage{graphicx}
\usepackage{natbib}
\usepackage{bibunits}
\usepackage{caption}
\usepackage{algorithm}
\usepackage{algorithmic}
\usepackage{amsmath}
\usepackage{multirow}
\usepackage{amssymb}
\usepackage{amsthm}
\usepackage{makecell}
\nocopyright

\usepackage{newfloat}
\usepackage{listings}
\DeclareCaptionStyle{ruled}{labelfont=normalfont,labelsep=colon,strut=off}
\floatstyle{ruled}
\newfloat{listing}{tb}{lst}{}
\floatname{listing}{Listing}
\usepackage{booktabs}

\newtheorem{lemma}{Lemma}
\newtheorem{proposition}{Proposition}
\theoremstyle{remark}

\defaultbibliographystyle{aaai2027}
\defaultbibliography{ref}

\makeatletter
\let\MergedOriginalMaketitle\maketitle
\let\MergedOriginalAtMaketitle\@maketitle
\let\MergedOriginalTitle\title
\let\MergedOriginalAuthor\author
\let\MergedOriginalDate\date
\let\MergedOriginalThanks\thanks
\let\MergedOriginalAffiliations\affiliations
\@ifundefined{corresponding}{}{\let\MergedOriginalCorresponding\corresponding}
\newcommand{\restoremergedtitlecommands}{%
  \let\maketitle\MergedOriginalMaketitle
  \let\@maketitle\MergedOriginalAtMaketitle
  \let\title\MergedOriginalTitle
  \let\author\MergedOriginalAuthor
  \let\date\MergedOriginalDate
  \let\thanks\MergedOriginalThanks
  \let\affiliations\MergedOriginalAffiliations
  \@ifundefined{MergedOriginalCorresponding}{}{\let\corresponding\MergedOriginalCorresponding}%
}
\makeatother 

\title{Credit the Right Box: Marginal Contribution Assignment for\\ Structured Visual Perception}
\author{
    Xinheng Han\textsuperscript{\rm 1,\rm 2}\thanks{Work done during internship at Amap, Alibaba Group.}, 
    Jianfei Wang\textsuperscript{\rm 2},
    Yu Chen\textsuperscript{\rm 2},
    Xiang Wang\textsuperscript{\rm 2}\thanks{Project Leader}
    Shuai Li\textsuperscript{\rm 2},
    Weixing Li\textsuperscript{\rm 1},
    Feng Pan\textsuperscript{\rm 1}\corresponding
}
\affiliations{
    \textsuperscript{\rm 1}School of Automation, Beijing Institute of Technology\\
    \textsuperscript{\rm 2}Amap, Alibaba Group\\
    \{hanxinheng, panfeng\}@bit.edu.cn
}

\begin{document}

\begin{bibunit}
\maketitle

\begin{abstract}

Multimodal Large Language Models (MLLMs) are increasingly expected to solve structured perception tasks that require visual recognition, language-to-object binding, object cardinality preservation, and precisely localized grounding and segmentation outputs. However, existing group-relative reinforcement learning methods provide only response-level supervision, creating a granularity mismatch for structured multi-object prediction: a single advantage is broadcast to all tokens in a response, without distinguishing individual box contributions. To address this mismatch, we propose \textbf{MCR-GRPO}, a marginal contribution assignment framework that derives box-level credit directly from each sampled response. Specifically, \textbf{Marginal Contribution Reward (MCR)} estimates each predicted box’s contribution through a leave-one-out comparison, measuring how the matched set value changes when the box is removed from the response. After within-response normalization, records that improve the set value receive positive credit, while redundant or harmful ones are suppressed. To make marginal attribution stable and informative, we further introduce a \textbf{Continuous Matched Set Value Evaluator} that integrates permutation-invariant matching, count-aware normalization, and graded localization. MCR-GRPO maps normalized box-level marginal advantages to the token spans that generated each box, preserving GRPO’s response-level comparison while enabling box-aware optimization of structured multi-object grounding. Experiments across REC, DOD, segmentation, and counting benchmarks show state-of-the-art performance over prior GRPO-based baselines.

\end{abstract}

\section{Introduction}

Multimodal Large Language Models (MLLMs)~\citep{llava, qwen25vl} have made rapid progress in general visual understanding, but many perception problems require more than a fluent textual answer. In grounding, segmentation, and counting tasks, a model cannot merely judge whether the queried objects are present, but must organize its prediction into a set of structured object records, each of which identifies a distinct target instance, preserves its role in the required cardinality, and associates it with a precise spatial output such as a point, box, or mask prompt. This requirement motivates structured visual perception, where MLLMs are expected to produce object-level visual records that jointly align language, instance identity, cardinality, and spatial localization.

Recent work has explored Group Relative Policy Optimization (GRPO)~\citep{grpo}, a reinforcement learning method, to enhance the structured visual perception ability of MLLMs~\citep{univgr1, vlmr1}. Representative methods such as VisionReasoner~\citep{visionreasoner} adopt GRPO to train models to generate structured point-and-box outputs, and further combine these outputs with SAM2~\citep{sam2} to unify detection~\citep{perceptionr1}, segmentation~\citep{samr1}, and counting~\citep{pixmocount} within a single perception framework.
 
However, GRPO optimizes sampled responses with a single response-level advantage, which is effective for selecting better answers at the response level but is too coarse for structured perception. A structured perception response is not a monolithic decision, since one response may contain a correctly localized object, a duplicated prediction, a missing instance, and a near-miss box. Broadcasting one response-level advantage to all generated tokens therefore creates a mismatch between the reward signal and the object record structure of the actual prediction, because the model is not told which individual box or object record improves the structured output and which one degrades it. As a result, correct object spans and harmful object spans can be reinforced together.

\begin{figure*}[t]
\centering
\includegraphics[width=0.8\textwidth]{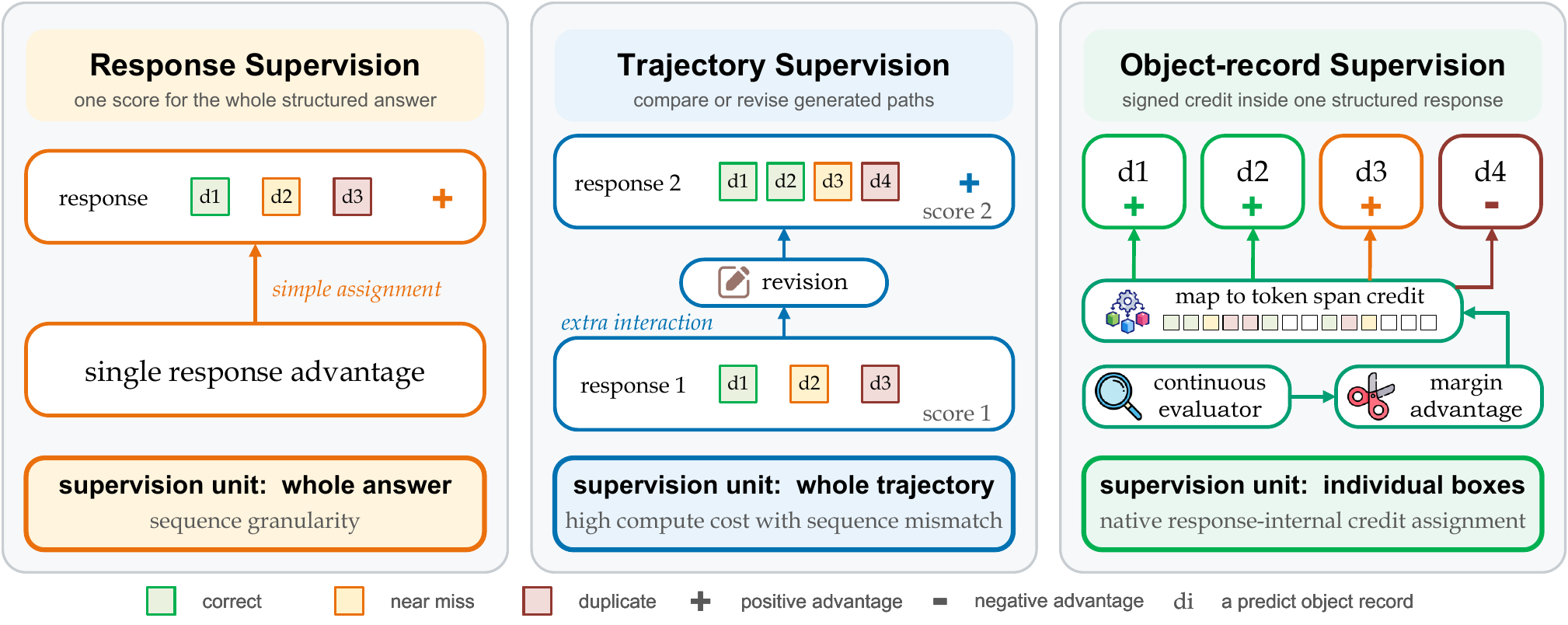}
\caption{\textbf{Three Paradigms of Supervision for Structured Visual Perception.} Response supervision scores the whole structured answer with a single advantage. Trajectory supervision revises entire generated paths, recovering feedback at the cost of extra rollouts, and still updates at sequence granularity. Object-record Supervision (MCR-GRPO) supervises at the object-record level, providing native response-internal credit assignment without auxiliary trajectories or extra rollouts.} \label{fig_compete}
\end{figure*}

To resolve this mismatch, the core question is whether object-local credit can be derived directly from the original structured response, without introducing additional rollouts, auxiliary trajectories, or a separate box-level objective. We answer this question with \textbf{MCR-GRPO}, a GRPO framework that assigns credit to individual boxes by reusing the matched set value of each response. Its core component, \textbf{Marginal Contribution Reward (MCR)}, estimates each prediction's contribution through a leave-one-out comparison. Given a structured response, MCR treats it as a set of object records, evaluates the matched set value of the whole set, removes one predicted record at a time, and measures how the matched set value changes. After within-response normalization, this produces signed, box-specific credit, where beneficial boxes receive positive credit, while redundant, distracting, or harmful boxes receive negative credit.

To make these marginal contributions smoother and more informative, we further introduce a \textbf{Continuous Matched Set Value Evaluator} that respects the unordered and count-sensitive nature of structured multi-object prediction. It computes pairwise scores between predicted and ground-truth objects, solves a Hungarian matching, and normalizes the matched score by the larger set size. It also incorporates continuous IoU scores rather than binarized localization labels, reducing the instability caused by borderline boxes near a hard threshold. The resulting value provides a permutation-invariant, count-aware, and localization-sensitive substrate for measuring marginal contribution. These marginal credits are normalized within each response and mapped to the token spans that generated the corresponding boxes. In this way, MCR-GRPO preserves GRPO's response-level comparison while enabling box-aware optimization for unified structured multi-object grounding.

The main contributions are summarized as follows:

\begin{itemize}

\item{\textbf{We propose MCR-GRPO, a GRPO framework for box-level credit assignment,} which derives response-internal credit without auxiliary trajectories while preserving response-level comparison.}

\item{\textbf{We introduce Marginal Contribution Reward (MCR) to estimate signed box-specific contributions} by measuring the leave-one-out change in set value, capturing duplication and substitution effects that independent pairwise scoring cannot express.}

\item{\textbf{We design a Continuous Matched Set Value Evaluator,} providing a smooth substrate that makes leave-one-out marginal attribution stable and informative.}

\item{\textbf{Experiments on REC, DOD, segmentation, and counting} demonstrate that MCR-GRPO achieves state-of-the-art performance over prior methods in unified structured visual perception, while preserving general VQA ability.}
    
\end{itemize}

\begin{figure*}[t]
\centering
\includegraphics[width=\textwidth]{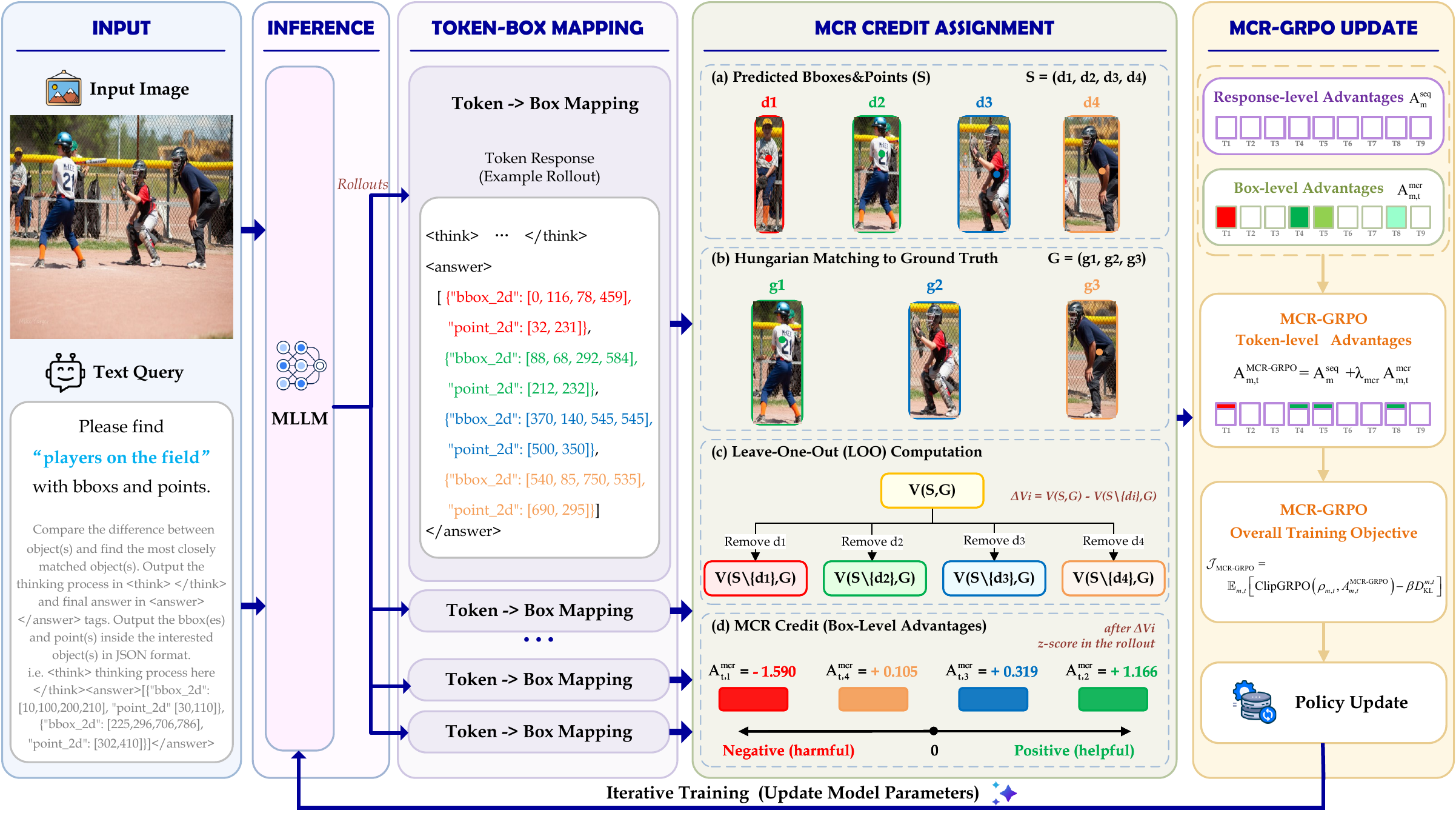}
\caption{\textbf{Overview of MCR-GRPO.} Given an image-query pair, the policy samples structured responses and parses each response into box-indexed object records with their generated token spans. MCR-GRPO matches the predicted records to ground-truth objects, computes the matched set value, and estimates each object-record's leave-one-out contribution to box-level MCR advantages. The advantages are mapped back to the token spans and combined with the response-level GRPO advantage, reinforcing helpful records while suppressing harmful or redundant ones.} \label{fig_main}
\end{figure*}

\section{Related Work}

\paragraph{Structured Visual Perception with MLLMs.}
Structured visual perception has progressed from language-conditioned detection, grounding, and segmentation to MLLMs that directly express boxes, regions, points, or masks in a language interface~\citep{glip,groundingdino,kosmos2,groma}. Pixel-level and promptable MLLMs further connect language reasoning to dense masks through special tokens, visual prompts, and segmentation decoders~\citep{reasonseg,segllm,read,sam,sam2}. Recent unified models extend this interface across grounding, segmentation, OCR, counting, and VQA-style perception by sharing backbones, output formats, or task mixtures~\citep{florence2,qwen25vl,llavaov}. The remaining bottleneck is the structure of the prediction itself: a response is an unordered, count-sensitive object set where correct instances, duplicates, omissions, and near-threshold boxes may coexist~\citep{grefcoco,d3}. Thus, beyond enabling MLLMs to emit coordinates, structured visual perception requires training signals that can identify which object records preserve instance identity, cardinality, and localization quality.

\paragraph{Learning Structured Perception.}
Supervised fine-tuning (SFT) teaches output formats and visual tool interfaces, but strong structured-perception systems still rely on large curated grounding, mask, and instruction mixtures~\citep{llava,pixellm,groma}. Reinforcement fine-tuning (RFT) reduces this dependence by replacing dense annotations with verifiable rewards, and recent visual RL methods use IoU, point, mask, format, or external-segmenter feedback for grounding, detection, segmentation, OCR, and counting~\citep{visualrft,vlmr1,univgr1,perceptionr1,segzero,samr1,segr1}. VisionReasoner advances toward unified visual perception, but it still follows GRPO's response-level supervision: one scalar reward and one relative advantage update the whole sampled response, even when the reward is computed by multi-object matching~\citep{visionreasoner}. This granularity mismatch can reinforce both correct and harmful boxes within the same structured answer.

Trajectory feedback provides another route by scoring intermediate steps, revised candidates, or generated paths~\citep{lets,mathshepherd}. Group Revision adapts this idea to object grounding by revising an initial response and using revision gains for reward shaping and advantage scaling~\citep{grouprevision}. It recovers supervision on hard cases, but its feedback unit is still a revised trajectory rather than an individual box, and the extra revision prompts and rollouts increase training cost. MCR-GRPO instead keeps the original GRPO responses and derives response-internal box-level credit from each predicted box's leave-one-out marginal contribution to the matched set value, without auxiliary trajectories, extra rollouts, or a separate box-level objective.

\section{Method}

\subsection{Problem Formulation}

Given an image $x\in\mathcal{X}$ and a language query $q\in\mathcal{Q}$, the policy model $\pi_\theta$ is optimized to generate a structured visual perception response. The ground truth is an unordered object set $G=\{g_j\}_{j=1}^{N}$, where $N$ is the number of target objects. Each ground-truth object $g_j$ contains spatial supervision, such as a bounding box $b_j^\star\in\mathbb{R}^{4}$ and a mask $m_j^\star$.

During RL training, multiple responses are sampled for the same image-query pair. Following GRPO, the old policy $\pi_{\theta_{\mathrm{old}}}$ samples a group of responses $\{y_m\}_{m=1}^{M}\sim \pi_{\theta_{\mathrm{old}}}(\cdot\mid x,q)$, where $M$ is the group size. For each sampled response, we omit the response index when discussing its internal predicted set. A deterministic parser extracts the predicted object set $S=\{d_i\}_{i=1}^{K}$, where $K$ is the number of predicted object records. Each predicted record $d_i$ contains a predicted bounding box $\hat b_i\in\mathbb{R}^{4}$, a predicted point $\hat p_i\in\mathbb{R}^{2}$, and the generated token span $\tau_i$ that produces this record.

A conventional response-level reward $R^{\mathrm{seq}}_m$ is assigned to each sampled response $y_m$, and the group-normalized advantage is computed as:
\begin{equation}
A^{\mathrm{seq}}_m = \frac{R^{\mathrm{seq}}_m-\mu_{x,q}}{\sigma_{x,q}+\epsilon},
\end{equation}
where $\mu_{x,q}$ and $\sigma_{x,q}$ are the mean and standard deviation of the response rewards within the sampled group. This advantage compares different responses, but it is applied uniformly to all valid tokens in the same response. Such response-level optimization cannot distinguish the object records inside one structured answer, even though different records may be correct, duplicated, mislocalized, or harmful. Therefore, our goal is to augment response-level GRPO with response-internal credit assignment over the predicted object set $S$.

\subsection{Continuous Matched Set Value Evaluator}

To measure the quality of a structured prediction, we introduce a continuous matched set value evaluator. Given a predicted object set $S=\{d_i\}_{i=1}^{K}$ and the ground-truth set $G=\{g_j\}_{j=1}^{N}$, the evaluator produces a scalar value $V(S,G)$ that reflects set-level grounding quality.

For each predicted record $d_i$ and ground-truth object $g_j$, we compute a continuous pair score:
\begin{equation}
\begin{aligned}
s(d_i,g_j) = \alpha s_{\mathrm{iou}}(\hat b_i,b_j^\star) &+ \beta s_{\mathrm{box}}(\hat b_i,b_j^\star) \\
& + \gamma s_{\mathrm{point}}(\hat p_i,m_j^\star).
\end{aligned}
\end{equation}
where $s_{\mathrm{iou}}$ measures box overlap, $s_{\mathrm{box}}$ measures coordinate-level localization closeness, and $s_{\mathrm{point}}$ measures point quality. The weights $\alpha,\beta,\gamma$ control the relative importance of these spatial cues.

Unlike prior binary reward designs that threshold localization correctness, our evaluator keeps the spatial quality continuous whenever possible. Specifically, the IoU term is directly defined as: $s_{\mathrm{iou}}(\hat b_i,b_j^\star)=\operatorname{IoU}(\hat b_i,b_j^\star)$.

For coordinate-level box quality, we first compute the mean absolute coordinate distance in pixel space: $\ell^{\mathrm{box}}_{ij}=\frac{1}{4}\left\|\hat b_i-b_j^\star\right\|_1$, and convert it into a truncated exponential score:
\begin{equation}
s_{\mathrm{box}}(\hat b_i,b_j^\star)
=
\begin{cases}
\exp(-\ell^{\mathrm{box}}_{ij}/10), & \ell^{\mathrm{box}}_{ij}\le 40,\\
0, & \text{otherwise}.
\end{cases}
\end{equation}

For point quality, we do not use the distance to a reference point as the main criterion. In the downstream SAM-based segmentation setting, a point is useful when it lies on the target object, while a point closer to a reference point is not necessarily better. Therefore, we define the point score using the ground-truth object mask $m_j^\star$ and bounding box $b_j^\star$:
\begin{equation}
s_{\mathrm{point}}(\hat p_i,m_j^\star)=
\begin{cases}
1.0, & \hat p_i \in b_j^\star \text{ and } \hat p_i \in m_j^\star,\\
0.3, & \hat p_i \in b_j^\star \text{ and } \hat p_i \notin m_j^\star,\\
0, & \text{otherwise}.
\end{cases}
\end{equation}

Among the three terms, IoU provides the primary estimate of region-level set quality, while the box-distance and point-validity terms provide auxiliary localization and promptability cues. We therefore set $\alpha=2$, $\beta=\gamma=1$ in our evaluator.

Since both $S$ and $G$ are unordered sets, we use Hungarian matching to compute an optimal one-to-one assignment between predicted records and ground-truth objects:
\begin{equation}
M^\star = \arg\max_{M'\in\mathcal{M}(S,G)} \sum_{(i,j)\in M'}s(d_i,g_j),
\end{equation}
where $\mathcal{M}(S,G)$ denotes the set of valid one-to-one matchings. In implementation, this maximum-weight bipartite matching is solved by the Hungarian algorithm. 

The continuous matched set value is defined as:
\begin{equation}
V(S,G) = \frac{1}{\max(K,N)} \sum_{(i,j)\in M^\star}s(d_i,g_j).
\end{equation}

This evaluator is permutation-invariant since it depends on Hungarian matching rather than object order. It is count-aware because the denominator penalizes extra predictions and missing targets and is continuous because near-correct boxes can receive partial credit instead of being collapsed into a binary failure. This value $V(S,G)$ provides the set-level basis for computing marginal contribution rewards.

\subsection{Marginal Contribution Reward}

Given the continuous matched set value $V(S,G)$, we estimate response-internal object credit through leave-one-out comparison. The core idea is to measure how the set-level value changes when one predicted object record is removed from the structured prediction. For each predicted record $d_i\in S$, we compute the leave-one-out value difference as:
\begin{equation}
\Delta V_i
=
V(S,G)-V(S\setminus\{d_i\},G).
\end{equation}

Here, $\Delta V_i$ is the raw marginal contribution of $d_i$ before normalization. When $\Delta V_i\gg 0$, removing $d_i$ substantially decreases the set value, so $d_i$ is a helpful record that contributes to the structured prediction. When $\Delta V_i\ll 0$, removing $d_i$ improves the set value, so $d_i$ is harmful, often due to duplication, false positives, count inflation, or misleading localization. This can occur only when the response over-predicts ($K>N)$, whereas for$K \le N$ every raw contribution is non-negative and poorly localized records are instead suppressed in relative terms after the normalization. When $\Delta V_i\approx 0$, $d_i$ has little marginal effect under the current evaluator, either because it is weakly matched or because its contribution is redundant with other records.

\begin{table*}[t]
    \centering
    \small
    \begin{tabular}{l|cccccccccc|c|ccc}  
        \toprule
        \multirow{3}{*}{Method} & \multicolumn{10}{c|}{Referring Expression Comprehension (Acc@0.5)} & \multirow{3}{*}{Avg.} & \multicolumn{3}{c}{DOD (AP)}\\

        & \multicolumn{2}{c}{ReasonG} & \multicolumn{3}{c}{RefCOCO} & \multicolumn{3}{c}{RefCOCO+} & \multicolumn{2}{c|}{RefCOCOg} &  & \multicolumn{3}{c}{$\mathrm{D^3}$}\\
                           &      Val      &     Test      &      Val      &     TestA     &     TestB     &      Val      &     TestA     &     TestB     &      Val      &     Test      &               &     Full      &     Pres.     &      Abs.     \\
        \midrule
        \multicolumn{15}{l}{\textit{Supervised Fine-Tuning Based Methods}}\\
        PerceptionGPT-7B   &       -       &       -       &      88.6     & \textbf{92.5} &      84.6     &      82.1     &      88.6     &      74.2     &      84.1     &      85.2     &       -       &       -       &       -       &       -       \\
        VistaLLM-7B        &       -       &       -       &      88.1     &      91.5     &      83.0     &      82.9     & \textbf{89.8} &      74.8     &      83.6     &      84.4     &       -       &       -       &       -       &       -       \\
        Elysium-7B         &       -       &       -       &      89.1     &      92.1     &      85.0     &      82.9     &      88.9     &      75.6     &      82.9     &      83.6     &       -       &       -       &       -       &       -       \\
        Groma-7B           &       -       &       -       &      89.5     &      92.1     & \textbf{86.3} &      83.9     &      88.9     &      78.1     &      86.3     &      87.0     &       -       &      16.0     &      15.9     &      16.3     \\
        \midrule
        \multicolumn{15}{l}{\textit{Open-Source MLLM and Reinforcement Fine-Tuning Based Methods}}\\
        Qwen2.5-VL-7B      &      68.9     &      59.8     &      88.8     &      91.7     &      81.4     &      82.3     &      88.2     &      69.2     &      84.7     &      85.7     &      80.1     &      19.6     &      19.4     &      20.3     \\
        SegZero-7B         &      69.3     &      64.6     &      89.3     &      91.5     &      81.9     &      82.0     &      87.6     &      74.7     &      86.1     &      86.3     &      81.3     &       -       &       -       &       -       \\
        VisionReasoner-7B  &      80.1     &      78.5     &      88.6     &      90.6     &      84.7     &      83.6     &      87.9     &      80.2     &      86.1     &      87.5     &      84.8     &      22.0     &      21.2     &      24.1     \\
        GroupRevision-7B   &      83.7     &      81.2     &      89.5     &      91.6     &      85.1     &      84.4     &      88.3     &      77.6     &      88.4     &      88.3     &      85.8     &       -       &       -       &       -       \\
        \midrule
        MCR-GRPO (ours)    & \textbf{84.0} & \textbf{83.3} & \textbf{90.2} &      91.6     &      86.2     & \textbf{85.2} &      87.8     & \textbf{80.6} & \textbf{89.1} & \textbf{89.9} & \textbf{86.8} & \textbf{23.1} & \textbf{22.5} & \textbf{24.8}\\
        \bottomrule
    \end{tabular}
    \caption{\textbf{Comparison with SOTA Methods on REC and DOD Tasks.} The best results are highlighted in \textbf{bold}.}
    \label{tab_RECDOD}
\end{table*}

This leave-one-out formulation captures interactions that independent pairwise scoring cannot. A duplicated prediction may obtain a reasonable local matching score, but it can still reduce the set value by increasing the predicted count. Conversely, a near-correct prediction can receive positive marginal credit if it improves the optimal matching. Thus, $\Delta V_i$ attributes each record according to its effect on the whole predicted set $S$, rather than its isolated similarity to a single ground-truth object. A single Hungarian matching costs \(O(\max(K,N)^3)\). Since MCR computes one full-set value and \(K\) leave-one-out values, the per-response matching cost is \(O((K+1)\max(K,N)^3)\), which is small compared with VLM rollout generation and policy optimization.

The raw sign of $\Delta V_i$ is regime-dependent. Under over-prediction $(K > N)$, an unmatched record provably receives strictly negative raw credit, so duplicate suppression follows from the count-aware normalization rather than a hand-designed penalty. Under $K \leq N$ every raw contribution is non-negative, so \textit{harmful} must be defined relative to a reference point. Absolute set quality is already carried by $V(S,G)$ through the response-level reward, and centering makes MCR zero-sum and hence orthogonal to that signal. The z-score is strictly increasing, so it preserves the ranking by marginal contribution, and unmatched or perfectly substitutable records provably occupy the lowest credits. Finally, $\left|\bar{A}_i^{\mathrm{mcr}}\right| \leq \sqrt{K-1}$, so the composite advantage keeps the sign of $A^{\mathrm{seq}}$ whenever $\lambda_{\mathrm{mcr}}\sqrt{K-1} < \left|A^{\mathrm{seq}}\right|$. MCR therefore guarantees a strict within-response ordering rather than an absolute sign per box, which is analyzed in \textbf{Appendix}.

We then normalize the raw value differences within the same rollout. This normalization is computed only over the predicted object records in $S$; tokens outside the structured object records, such as reasoning tokens or other non-structural text, do not participate in this normalization. For $K>1$, we compute:
\begin{equation}
\mu_{\Delta}=\frac{1}{K}\sum_{i=1}^{K}\Delta V_i, \quad \sigma_{\Delta}=\sqrt{\frac{1}{K}\sum_{i=1}^{K}(\Delta V_i-\mu_{\Delta})^2},
\end{equation}
and obtain the normalized object-level MCR credit as:
\begin{equation}
\bar A_i^{\mathrm{mcr}}=\frac{\Delta V_i-\mu_{\Delta}}{\sigma_{\Delta}+\epsilon}.
\end{equation}

The normalization is rollout-local, so MCR compares object records inside the same structured response rather than across different sampled responses.

For the special case where a rollout contains only one predicted object record, there is no response-internal comparison to perform. We therefore set its normalized MCR credit to zero: $\bar A_1^{\mathrm{mcr}}=0$, if $K=1$. With this convention, the normalized object-level MCR credits within each rollout have zero algebraic sum $\sum_{i=1}^{K}\bar A_i^{\mathrm{mcr}}=0$. MCR is thus a zero-sum redistribution at the object-record level. It injects no net preference for or against the whole response, but reallocates credit among the records inside it, while the ranking across rollouts remains determined by the response advantage $A^{seq}$. After mapping to token spans, exact cancellation holds at the record level rather than the token level, since spans differ in length; in practice structured records have near-uniform span lengths.

Finally, the normalized MCR credit is delivered to the tokens that generate the corresponding object record. Let $\tau_i$ denote the generated token span of $d_i$. We define the token-level MCR term for record $d_i$ as:
\begin{equation}
A_{t,i}^{\mathrm{mcr}}
=
\begin{cases}
\bar A_i^{\mathrm{mcr}}, & t\in\tau_i,\\
0, & t\notin\tau_i.
\end{cases}
\end{equation}

Thus, only the tokens belonging to the structured span of $d_i$ receive its MCR credit. Tokens that are not assigned to any predicted object record have zero MCR term and are affected only by the response-level GRPO advantage in the final policy update.

\subsection{MCR-GRPO Optimization}

MCR-GRPO combines the response-level GRPO signal with the token-span MCR residual. For each sampled response $y_m$, we first compute its response-level reward. Let $F_m\in\{0,1\}$ denote whether $y_m$ satisfies all required structural rules, including the reasoning tags, answer tags, valid JSON format, and valid $bbox\_2d$ and $point\_2d$ fields. These checks follow the format-reward design in VisionReasoner, but we use an all-pass rule: if any condition fails, $F_m=0$.

The response-level reward is:
\begin{equation}
R_m^{\mathrm{seq}} = F_m \left(4+1.5\,R_m^{\mathrm{nr}}+ V(S,G)\right),
\end{equation}
where $R_m^{\mathrm{nr}}\in\{0,1\}$ is the non-repetition reward, and $V(S,G)$ is the accuracy reward computed by the continuous matched set value evaluator. Since $\alpha=2,\beta=\gamma=1$, the maximum value of $V(S,G)$ is $4$. The constant 4 provides a fixed base reward for format-valid responses so that, after group normalization, any response failing the structural checks receives a strongly negative advantage. Following VisionReasoner, the non-repetition reward $R_m^{\mathrm{nr}}$ is set to 1 unless the response contains repeated predictions.

We then obtain the response-level advantage $A_m^{\mathrm{seq}}$ by applying the group z-score normalization in Eq.~(1) to $R_m^{\mathrm{seq}}$. The final token-level advantage is:
\begin{equation}
A_{m,t}^{\mathrm{MCR\text{-}GRPO}} = A_m^{\mathrm{seq}} + \lambda_{\mathrm{mcr}}A_{t,i}^{\mathrm{mcr}}, \quad t\in\tau_i,
\end{equation}
where $A_m^{\mathrm{seq}}$ preserves GRPO's response-level comparison, while $A_{t,i}^{\mathrm{mcr}}$ adds localized residual credit only to the token spans that generate predicted object records.

We optimize the policy with a clipped GRPO objective. The token-level likelihood ratio is:
\begin{equation}
\rho_{m,t}(\theta)=\frac{\pi_\theta(y_{m,t}\mid x,q,y_{m,<t})}{\pi_{\theta_{\mathrm{old}}}(y_{m,t}\mid x,q,y_{m,<t})}.
\end{equation}

\begin{table*}[t]
    \centering
    \small
    \begin{tabular}{l|ccccccc|c|ccc|c}
        \toprule
        \multirow{3}{*}{Method} & \multicolumn{7}{c|}{Segmentation (gIoU \& cIoU)} & \multirow{3}{*}{Avg.} & \multicolumn{3}{c|}{Counting (Acc)} & \multirow{3}{*}{Avg.} \\
        
        & \multicolumn{2}{c}{ReasonSeg} & \multicolumn{2}{c}{RefCOCO} & \multicolumn{2}{c}{RefCOCO+} & RefCOCOg &  & \multicolumn{2}{c}{Pixmo} & Count &  \\
        
                           & Val   & Test  & TestA & TestB & TestA & TestB & Test  &       &  Val  & Test  & Test  &       \\
        \midrule
        \multicolumn{13}{l}{\textit{Supervised Fine-Tuning Based Methods}}\\
        LLaVA-OV-7B        &       -       &       -       &      58.1     &       -       &      47.1     &       -       &      55.6     &       -       &       -       &       -       &       -       &       -       \\
        LISA-7B            &      44.4     &      36.8     &      79.1     &      72.3     &      70.8     &      58.1     &      70.6     &       -       &       -       &       -       &       -       &       -       \\
        PixelLM-7B         &       -       &       -       &      78.6     &      68.2     &      71.7     &      58.3     &      70.5     &       -       &       -       &       -       &       -       &       -       \\
        PerceptionGPT-7B   &       -       &       -       &      78.6     &      71.7     &      73.9     &      61.3     &      71.7     &       -       &       -       &       -       &       -       &       -       \\
        SEGLLM             &      57.2     &      52.4     & \textbf{81.5} & \textbf{75.4} &      73.0     & \textbf{62.5} & \textbf{73.6} &      67.9     &       -       &       -       &       -       &       -       \\
        Read-7B            &      59.8     &      56.8     &      80.2     &      73.2     &      73.7     &      60.4     &      71.4     &      67.9     &       -       &       -       &       -       &       -       \\
        \midrule
        \multicolumn{13}{l}{\textit{Open-Source MLLM and Reinforcement Fine-Tuning Based Methods}}\\
        Qwen2.5-VL-7B      &      56.9     &      52.1     &      77.9     &      66.5     &      74.0     &      55.6     &      70.9     &      64.8     &      63.3     &      67.9     &      76.0     &      69.1     \\
        Seg-R1-7B          &      58.6     &      56.7     &      78.7     &      67.6     &      70.9     &      57.9     &      71.4     &      66.0     &       -       &       -       &       -       &       -       \\
        Seg-Zero-7B        &      62.6     &      57.5     &      80.3     &      72.2     & \textbf{76.2} &      62.3     &      72.6     &      69.1     &       -       &       -       &       -       &       -       \\
        VisionReasoner-7B  &      66.3     &      63.6     &      77.4     &      67.6     &      71.1     &      55.8     &      68.3     &      67.2     &      70.1     &      69.5     &      87.6     &      75.7     \\
        GroupRevision-7B   &      67.5     &      66.7     &      78.0     &      69.5     &      73.3     &      59.3     &      71.1     &      69.4     & \textbf{75.9} &      73.0     &      91.0     &      80.0     \\
        \midrule
        MCR-GRPO (ours)    & \textbf{69.4} & \textbf{67.2} &      78.0     &      69.2     &      72.1     &      59.5     &      71.6     & \textbf{69.6} &     75.7     & \textbf{76.4} & \textbf{92.7} & \textbf{81.6} \\
        \bottomrule
    \end{tabular}
    \caption{\textbf{Comparison with SOTA on Segmentation and Counting Tasks.} We use SAM2 if necessary in segmentation tasks. The best results are highlighted in \textbf{bold}.}
    \label{tab_SegCount}
\end{table*}

For compact notation, we define the clip and KL terms:
\begin{equation}
\begin{aligned}
L_{m,t}(\theta)&=\min\big(\rho_{m,t}(\theta)A_{m,t}^{\mathrm{MCR\text{-}GRPO}},\\
&\operatorname{clip}(\rho_{m,t}(\theta),1-\epsilon_c,1+\epsilon_c) A_{m,t}^{\mathrm{MCR\text{-}GRPO}} \big).
\end{aligned}
\end{equation}

\begin{equation}
D_{m,t}^{\mathrm{KL}} = D_{\mathrm{KL}} \left( \pi_\theta(\cdot\mid x,q,y_{m,<t}) \;\|\; \pi_{\mathrm{ref}}(\cdot\mid z,q,y_{m,<t}) \right).
\end{equation}

The final objective of MCR-GRPO is:
\begin{equation}
\begin{aligned}
\mathcal{J}_{\mathrm{MCR}}&_{\mathrm{\text{-}GRPO}}(\theta)=\\
&\mathbb{E} \left[ \frac{1}{M} \sum_{m=1}^{M} \frac{1}{|y_m|} \sum_{t=1}^{|y_m|} \left( L_{m,t}(\theta)-\beta D_{m,t}^{\mathrm{KL}} \right) \right].
\end{aligned}
\end{equation}

\section{Experiment}

\paragraph{Training Data.}
We train MCR-GRPO on VisionReasoner7K~\citep{visionreasoner}, a compact multi-object visual perception corpus comprising 7,099 object-record-supervised examples that span category-level localization, referring comprehension, and reasoning instructions. To avoid introducing additional supervision, we use the ground-truth boxes and points as SAM2 prompts to generate masks, rather than relying on human-annotated masks. These sources do not include the held-out DOD, counting, or VQA sets used below.

\paragraph{Evaluation Benchmarks.}
We evaluate the models across five task families: referring expression comprehension (REC), described object detection (DOD), segmentation, object counting, and visual question answering (VQA). For REC, we report results on RefCOCO(+/g)~\citep{refcoco} and additionally evaluate on ReasonG, which is derived by converting ReasonSeg~\citep{reasonseg} masks into bounding boxes. For described object detection, we use the $D^3$ benchmark~\citep{d3}. For segmentation, we evaluate referring segmentation on RefCOCO(+/g) and reasoning segmentation on ReasonSeg. For counting, we use PixMo-Count~\citep{pixmocount} and CountBench~\citep{countbench}. We further include standard VQA benchmarks as auxiliary probes to examine whether structured-perception training preserves general multimodal ability.

\paragraph{Implementation Details.}
We initialize the MLLM from Qwen2.5-VL-7B-Instruct~\citep{qwen25vl} and use pretrained SAM2~\citep{sam2} weights to generate segmentation masks when mask outputs are required. Training is implemented with the VeRL~\citep{verl} framework, and rollout generation is accelerated by vLLM~\citep{vllm}. Unless otherwise specified, we train with a learning rate of $1\times10^{-6}$, a KL coefficient of $5\times10^{-3}$, gradient accumulation of 2, and a global batch size of 16. For each prompt, we sample 8 rollouts for group-relative optimization.

\paragraph{Evaluation Metrics.}
Following standard evaluation protocols, we adopt Acc@0.5 for REC, where a prediction is correct if its prediction has an IoU of least 0.5. For DOD, we report the standard detection AP. We use generalized Intersection over Union (gIoU) for reasoning segmentation and cumulative IoU (cIoU) for referring segmentation. For counting, we report accuracy based on predicted box count.

\paragraph{REC and DOD.}
We compare MCR-GRPO with current state-of-the-art methods on REC and DOD tasks, including PerceptionGPT-7B~\citep{perceptiongpt}, VistaLLM-7B~\citep{vistallm}, Elysium-7B~\citep{elysium}, Groma-7B~\citep{groma}, Qwen2.5-VL-7B~\citep{qwen25vl}, SegZero-7B~\citep{segzero}, VisionReasoner-7B~\citep{visionreasoner}, and GroupRevision-7B~\citep{grouprevision}. As shown in Table~\ref{tab_RECDOD}, our method achieves the best average REC accuracy of $86.8\%$. On the DOD benchmark $\mathrm{D^3}$, MCR-GRPO also consistently outperforms previous methods, achieving the best AP with $23.1\%$, $22.5\%$, and $24.8\%$, respectively.

\begin{table}[!t]
    \centering
    \small
    \begin{tabular}{c|ccc}
        \toprule
        \multirow{2}{*}{Model}&   OCRBench  &    ChartQA    &   SimpleVQA   \\
                        &      Num      &       Acc     &      Acc      \\
        \midrule
        Qwen2.5-VL-7B   &    858      &      83.8     &      26.4     \\
        MCR-GRPO (ours) & \textbf{874} & \textbf{87.7} & \textbf{26.5} \\      
        \bottomrule
    \end{tabular}
    \begin{tabular}{c|ccc}
        \toprule
        \multirow{2}{*}{Model}&    DUDE   &    MMStar     & MME-Realworld \\
                          &      Acc      &      Acc      &   Acc (Lite)  \\
        \midrule
        Qwen2.5-VL-7B     &      47.9     &      62.0     &      43.3     \\
        MCR-GRPO (ours)   & \textbf{48.9} & \textbf{62.6} & \textbf{48.2} \\      
        \bottomrule
    \end{tabular}
    \caption{Evaluation results on VQA benchmarks.}
    \label{tab_vqa}
\end{table}
\begin{table}[t]
    \centering
    \small
    \begin{tabular}{cc|ccccc}
        \toprule
        \multirow{3}{*}{MCR} & \multirow{3}{*}{Cont} & \multicolumn{2}{c}{REC} & DOD & Seg & Count\\
                          &                   & \multicolumn{2}{c}{RefCOCOg}  & $\mathrm{D^3}$ &   ReasonSeg   &     Pixmo     \\
                          &                   &      Val      &      Test     &      Full      &      Val      &      Val      \\
        \midrule
        $\mathrm{\times}$ & $\mathrm{\times}$ &      86.1     &      87.5     &      22.0      &      66.3     &      70.1     \\
        \checkmark        & $\mathrm{\times}$ &      88.9     &      89.4     &      22.9      &      68.4     &      67.1     \\
        $\mathrm{\times}$ &     \checkmark    &      88.1     &      88.3     &      22.0      &      68.5     &      69.5     \\
        \checkmark        &     \checkmark    & \textbf{89.1} & \textbf{89.9} & \textbf{23.1}  & \textbf{69.4} & \textbf{75.7} \\        
        \bottomrule
    \end{tabular}
    \caption{\textbf{Ablation Study of Our Approach} on the VisionReasoner baseline with Marginal Contribution Reward (MCR) and Continuous Matched Set Value Evaluator (Cont).}
    \label{tab_ablation}
\end{table}

\paragraph{Segmentation and Counting.}
We further compare MCR-GRPO with state-of-the-art methods on segmentation and counting tasks, including LLaVA-OV-7B~\citep{llavaov}, LISA-7B~\citep{reasonseg}, PixelLM-7B~\citep{pixellm}, PerceptionGPT-7B~\citep{perceptiongpt}, SEGLLM~\citep{segllm}, Read-7B~\citep{read}, Qwen2.5-VL-7B~\citep{qwen25vl}, and Seg-R1-7B~\citep{segr1} besides the baselines discussed above. As shown in Table~\ref{tab_SegCount}, MCR-GRPO achieves the best average performance on both segmentation and counting, reaching $69.6\%$ on segmentation and $81.6\%$ on counting. These indicate that, by assigning marginal contribution rewards to box-level spatial structures, our method strengthens the model's ability to preserve object identity, distinguish target instances, and maintain count consistency without introducing additional trajectory.

\paragraph{Visual QA Ability.}
To evaluate whether MCR-GRPO preserves and improves general visual question answering ability, we compare it with Qwen2.5-VL-7B on six widely used VQA benchmarks, including DUDE~\citep{dude}, ChartQA~\citep{chartqa}, SimpleVQA~\citep{simplevqa}, OCRBench~\citep{ocrbench}, MMStar~\citep{mmstar}, and MME-Realworld-Lite~\citep{mmerealworld}. As shown in Table~\ref{tab_vqa}, MCR-GRPO consistently improves the model across all evaluated benchmarks. These results indicate that our training does not degrade the model's general VQA ability. Instead, by improving structured visual perception, MCR-GRPO brings consistent gains to VQA scenarios beyond the tasks directly optimized during training.

\paragraph{Ablation of the Key Components.}
We conduct ablation studies on four representative tasks, including REC, DOD, segmentation, and counting, to analyze the effect of the two key designs in MCR-GRPO, namely Marginal Contribution Reward (MCR) and the Continuous Matched Set Value Evaluator (Cont). We use VisionReasoner as the baseline. When Cont is removed, the value function used by MCR is VisionReasoner's original Accuracy Reward. Table~\ref{tab_ablation} disentangles the two designs. Adding the continuous evaluator alone improves REC ($+2.0\%$ / $+0.8\%$ on RefCOCOg) and reasoning segmentation ($+2.2\%$) by making response-level ranking smoother, but it leaves DOD unchanged and slightly lowers counting ($69.5\%$ vs. $70.1\%$): a smoother scalar is still broadcast over all boxes and cannot tell the model which box to drop. Applying MCR alone on top of the thresholded value improves REC and DOD, yet degrades counting below the baseline: on a discretized value, removing a box either leaves the score unchanged or changes it abruptly, so marginal contributions become sparse and noisy, and the resulting credit misleads cardinality-sensitive behavior. Only the combination improves all four tasks, with the largest margin exactly on counting ($+5.6\%$ over baseline and $+8.6\%$ over MCR alone). The two designs are therefore complementary rather than independently additive: the continuous matched set value is the substrate that makes leave-one-out attribution informative, and box-level attribution is the mechanism that converts graded set-level feedback into cardinality-aware learning.

\begin{table}[t]
    \centering
    \small
    \begin{tabular}{c|ccccc}
        \toprule
        \multirow{3}{*}{MCR Ratio}& \multicolumn{2}{c}{REC} & DOD & Seg & Count\\
             & \multicolumn{2}{c}{RefCOCOg}  & $\mathrm{D^3}$ &   ReasonSeg   &     Pixmo     \\
             &      Val      &      Test     &      Full      &      Val      &      Val      \\
        \midrule
        0.05 &      88.1     &      88.7     &      22.9      &      68.8     &      72.9     \\
        0.10 &      89.1     & \textbf{89.9} & \textbf{23.1}  & \textbf{69.4} &      75.7     \\
        0.20 & \textbf{89.6} &      89.6     &      22.9      &      67.0     & \textbf{75.9} \\       
        \bottomrule
    \end{tabular}
    \caption{Ablation of MCR Ratio ($\lambda_{\mathrm{mcr}}$).}
    \label{tab_mcr}
\end{table}

\paragraph{Ablation of MCR Ratio.}
We further study the effect of the MCR ratio ($\lambda_{\mathrm{mcr}}$), which controls the strength of the box-level MCR advantage in the final policy update. As shown in Table~\ref{tab_mcr}, a small ratio of 0.05 already brings competitive performance, but the MCR signal is not strong enough to fully exploit response-internal box-level credit. Increasing the ratio to 0.10 gives the best overall trade-off, achieving the highest REC accuracy, DOD AP, and Segmentation gIoU, while maintaining strong counting accuracy. When the ratio is further increased to 0.20, the model obtains slightly higher RefCOCOg Val and PixMo Val results, but its performance drops on $\mathrm{D^3}$ and ReasonSeg. This suggests that overly emphasizing box-level residual credit may weaken the balance with the original response-level GRPO objective. Therefore, we set the MCR ratio to 0.10 by default, as it provides a stable balance between response-level preference optimization and fine-grained object-record credit assignment. \textbf{Additional qualitative results} are provided in the \textbf{Appendix}.

\section{Conclusion}
This paper addresses the response-internal credit assignment problem in GRPO-based structured visual perception. We propose MCR-GRPO, a GRPO framework that assigns box-level marginal contribution rewards to structured object records while preserving the original response-level preference comparison. The core idea is to evaluate the predicted object set with a continuous matched set value and then estimate each box's contribution through leave-one-out comparison, so that helpful, redundant, and harmful object records can receive different learning signals. Experiments on REC, DOD, segmentation, counting, and VQA show that MCR-GRPO improves structured visual perception across multiple benchmarks. MCR-GRPO enables object-level visual records to serve as the basic units for learning, credit assignment, and multi-object understanding in MLLMs, which offers a new perspective for structured visual perception.

\renewcommand{\refname}{References}
\putbib[ref]
\end{bibunit}

\restoremergedtitlecommands
\title{Supplementary Material\\
       Credit the Right Box: Marginal Contribution Assignment\\
       for Structured Visual Perception}
\author{
    Xinheng Han\textsuperscript{\rm 1,\rm 2}\thanks{Work done during internship at Amap, Alibaba Group.}, 
    Jianfei Wang\textsuperscript{\rm 2},
    Yu Chen\textsuperscript{\rm 2},
    Xiang Wang\textsuperscript{\rm 2}\thanks{Project Leader}
    Shuai Li\textsuperscript{\rm 2},
    Weixing Li\textsuperscript{\rm 1},
    Feng Pan\textsuperscript{\rm 1}\corresponding
}
\affiliations{
    \textsuperscript{\rm 1}School of Automation, Beijing Institute of Technology\\
    \textsuperscript{\rm 2}Amap, Alibaba Group\\
    \{hanxinheng, panfeng\}@bit.edu.cn
}

\begin{bibunit}
\setcounter{footnote}{0}
\makeatletter
\let\c@aaai@eqfn\@undefined
\let\c@aaai@eqfn\@undefined
\let\c@aaai@corrfn\@undefined
\let\titlearea\@undefined
\let\actualheight\@undefined
\makeatother
\maketitle

\setcounter{secnumdepth}{2}
\setcounter{figure}{0}
\setcounter{table}{0}
\setcounter{equation}{0}
\appendix
\renewcommand{\theequation}{\thesection\arabic{equation}}
\renewcommand{\thetable}{\thesection\arabic{table}}
\renewcommand{\thefigure}{\thesection\arabic{figure}}

\section{Implementation Details}

We initialize the MLLM from Qwen2.5-VL-7B-Instruct~\citep{qwen25vl} and use pretrained SAM2~\citep{sam2} weights to generate segmentation masks when mask outputs are required. Training is implemented with the VeRL~\citep{verl} framework, and rollout generation is accelerated by vLLM~\citep{vllm}. We train on a single node with eight NVIDIA H20 GPUs, each with $141$~GB of memory, and set vLLM's gpu memory utilization to $0.6$.

We train for one epoch on the $7,099$-sample VisionReasoner7K~\citep{visionreasoner} training set with a rollout batch size of $16$ and drop last is true, resulting in $443$ steps. We empirically find that checkpoints around $400$ training steps are sufficiently optimized and yield stronger performance. Therefore, our evaluation mainly focuses on this training stage. Unless otherwise specified, we use AdamW as the optimizer, training with a learning rate of $1\times10^{-6}$, a KL coefficient of $5\times10^{-3}$, gradient accumulation of $2$, a global batch size of $16$, and gradient clipping with a maximum norm of $1.0$.

For each prompt, we sample $8$ rollouts for group-relative optimization. The maximum number of generated tokens per rollout is $2000$. Rollout sampling uses a temperature of $1.0$, top-$p$ of $1.0$, and disabled top-$k$ sampling.

Notably, all images are resized to $840 \times 840$ before inference in both training and testing.

\section{Reward and Evaluator Details}

Before computing the response-level reward, we first validate whether a sampled response can be deterministically parsed into structured object records. This validation is necessary because both the continuous matched set value and the MCR attribution require a well-defined predicted object set $S$. If the response format is invalid, the object records cannot be reliably extracted, and the subsequent matching-based reward computation becomes undefined.

A response is treated as structurally valid only when all required components are present. Specifically, it must contain a complete reasoning region and a complete answer region, delimited by \texttt{<think>...</think>} and \texttt{<answer>...</answer>}, respectively. The content inside \texttt{<answer>...</answer>} must be parseable as a JSON array, where each element is a dictionary representing one predicted object record. For every predicted object record, the \texttt{bbox\_2d} field must be present and contain exactly four numerical values, and the \texttt{point\_2d} field must be present and contain exactly two numerical values. These requirements ensure that each predicted item can be converted into a box-indexed object record with a valid point.

We use a conjunctive all-pass rule for format validation. Let $F_m\in\{0,1\}$ denote the structural validity indicator for the sampled response $y_m$. We set $F_m=1$ only if the reasoning-answer structure, JSON parsing, array structure, dictionary format, and all required object fields are valid simultaneously. If any requirement fails, we set $F_m=0$. This differs from partial format scoring: an invalid response does not receive separate partial rewards for individual fields, because a single malformed component can prevent the construction of the predicted object set $S$ and thus block the evaluator and MCR computation.

Given this structural gate, the response-level reward is defined as:
\begin{equation}
R_m^{\mathrm{seq}}
=
F_m
\left(
4 + 1.5 R_m^{\mathrm{nr}} + V(S,G)
\right),
\end{equation}
where $R_m^{\mathrm{nr}}\in\{0,1\}$ is the non-repetition reward and $V(S,G)$ is the continuous matched set value. Following VisionReasoner, $R_m^{\mathrm{nr}}$ is set to 1 unless the response contains repeated predictions.

The constant 4 serves two purposes. First, it creates a clear reward separation between structurally valid and invalid responses after group-relative normalization, so responses that cannot be parsed receive a strongly negative advantage. Second, it aligns the base reward scale with the accuracy term. Since the evaluator uses $\alpha=2$ and $\beta=\gamma=1$, the maximum pair score is 4, and the matched set value satisfies $V(S,G)\le 4$. Therefore, the base reward has the same scale as the maximum accuracy reward, making format validity and structured prediction quality comparable within the response-level reward. Thus, format-valid responses are rewarded on a scale that can still be refined by non-repetition and matched-set accuracy, while format-invalid responses receive zero response-level reward and no object-level MCR attribution.

\section{Protocol and Prompt Comparison}
\label{app:inference-prompt-comparison}

\begin{table}[t]
\centering
\small
\begin{tabular}{lccc}
\toprule
Method         &  Supervision  & Stages & Test-Time\\
\midrule
Qwen2.5-VL     &   base MLLM   &    1   & Initial         \\
VisionReasoner &    Response   &    1   & Initial Struct  \\
GroupRevision  &   Trajectory  &    2   & Revised Struct  \\
MCR-GRPO       & Object-record &    1   & Initial Struct  \\
\bottomrule
\end{tabular}
\caption{\textbf{Comparison of Supervision Strategies and Inference Stages.} VisionReasoner uses response supervision, GroupRevision introduces trajectory supervision through an additional revision pass, and MCR-GRPO performs object-record supervision within a single structured response.}
\label{tab:app-strategy-comparison}
\end{table}

\begin{table}[t]
\centering
\small
\begin{tabular}{|c|}
\toprule
\textbf{Qwen2.5-VL Prompt} \\
\midrule
Locate "\{\textit{query}\}", report the bboxes coordinates\\ in JSON format.\\
\midrule
\textbf{MCR-GRPO \& VisionReasoner Prompt}\\
\midrule
Please find “\{\textit{query}\}” with bboxs and points.\\
Compare the difference between object(s) and find the most\\
closely matched object(s). Output the thinking process in\\
\texttt{<think> </think>} and final answer in \texttt{<answer> }\\
\texttt{</answer>} tags. Output the bbox(es) and point(s) inside\\
the interested object(s) in JSON format.\\
i.e. \texttt{<think>} thinking process here \texttt{</think>}\\
\texttt{<answer>} \{\textit{example}\} \texttt{</answer>}\\
\midrule
\textbf{GroupRevision Stage} \textbf{\uppercase\expandafter{\romannumeral 1}} \textbf{(Init) Prompt}\\
\midrule
Please find "\{\textit{query}\}" with bounding boxes and points.\\
Compare the difference between object(s) and identify the\\
most closely matched one(s). Output the thinking process\\
in \texttt{<think> ... </think>} and the final answer in\\
\texttt{<answer> ... </answer>} tags. Return the bbox(es)\\
and point(s) of the referenced object(s) in JSON format.\\
i.e., \texttt{<think>} thinking process here \texttt{</think>}\\
\texttt{<answer>} \{\textit{example}\} \texttt{</answer>}\\
\midrule
\textbf{GroupRevision Stage} \textbf{\uppercase\expandafter{\romannumeral 2}} \textbf{(Revision) Prompt}\\
\midrule
You previously predicted bounding box(es) (\{\textit{box}\}) and\\
point(s) (\{\textit{points}\}) for the question "\{\textit{query}\}", with the\\
reasoning \{\textit{think}\}. Rethink whether the previous bbox(es)\\
and point(s) match the
target object(s).\\
– If yes: keep the same object(s); tighten each bbox to the\\
object boundary and set the point to the object center.\\
– If no: discard previous predictions and output the correct\\
object(s), each with one bbox and one point inside it.\\
Respond with exactly:\\
\texttt{<think>} thinking process here \texttt{</think>}\\
\texttt{<answer>} \{\textit{example}\} \texttt{</answer>}\\
\bottomrule
\end{tabular}
\caption{\textbf{Prompt Comparison Across Methods.} MCR-GRPO, Qwen2.5-VL, and VisionReasoner use a single initial prompt, while GroupRevision adds a second-stage revision prompt.}
\label{tab:app-prompt-comparison}
\end{table}

We compare the supervision strategies and inference protocols of different models. As shown in Table~\ref{tab:app-strategy-comparison}, Qwen2.5-VL~\citep{qwen25vl} serves as the base MLLM and directly generates the initial response. VisionReasoner~\citep{visionreasoner} follows response supervision, where a single response-level signal is assigned to the whole structured answer. GroupRevision~\citep{grouprevision} uses trajectory supervision by adding a second-stage revision pass conditioned on the initial prediction. In contrast, MCR-GRPO performs object-record supervision within the original structured response, assigning box-level credit during training while preserving a single-stage inference protocol.

Table~\ref{tab:app-prompt-comparison} lists the prompt templates used in our comparison. VisionReasoner and MCR-GRPO use the same structured-output prompt, while GroupRevision additionally uses a second-stage revision prompt.

\section{Additional Experiments on COCO}
\label{app:coco-additional-results}

We further evaluate MCR-GRPO on COCO to examine whether response-internal box-level credit assignment generalizes to a standard object detection benchmark. We compare our model with Qwen2.5-VL~\citep{qwen25vl}, VisionReasoner~\citep{visionreasoner}, and GroupRevision~\citep{grouprevision}. We evaluate on COCO val2017 using the standard COCO detection protocol. All methods are evaluated with the same image resizing, decoding configuration, prompt format, and post-processing rules. As shown in Table~\ref{tab:app-coco-main}, MCR-GRPO achieves the best COCO AP among the compared methods. Compared with Qwen2.5-VL, MCR-GRPO improves AP from $29.2\%$ to $39.7\%$. Compared with VisionReasoner, MCR-GRPO improves AP from $37.7\%$ to $39.7\%$. Notably, MCR-GRPO also slightly outperforms GroupRevision, while using only a single inference stage and requiring no additional revision pass.

\begin{table}[t]
\centering
\small
\begin{tabular}{lcc}
\toprule
Method            & Extra Stage       & COCO AP$_{50:95}$ \\
\midrule
Qwen2.5-VL-7B     & $\mathrm{\times}$ &      29.2         \\
VisionReasoner-7B & $\mathrm{\times}$ &      37.7         \\
GroupRevision-7B  &    \checkmark     &      39.2         \\
\midrule
MCR-GRPO (Ours)   & $\mathrm{\times}$ & \textbf{39.7}     \\
\bottomrule
\end{tabular}
\caption{\textbf{Additional COCO Results.} MCR-GRPO achieves the best COCO AP without introducing an extra inference stage.  The best results are highlighted in \textbf{bold}.}
\label{tab:app-coco-main}
\end{table}

To further compare inference protocols, Table~\ref{tab:app-coco-stage} reports detailed COCO metrics for GroupRevision and MCR-GRPO. GroupRevision uses a two-stage inference procedure, where the model first generates an initial response and then performs an additional revision stage. In contrast, MCR-GRPO keeps the inference pipeline single-stage: it derives box-level supervision during training, but does not require an additional response generation or revision trajectory at test time. Despite this simpler inference protocol, MCR-GRPO obtains higher AP, AP$_{50}$, AP$_{75}$, AR$_{10}$, and AR$_{100}$ than GroupRevision-7B.

\begin{table}[t]
\centering
\small
\begin{tabular}{lcccc}
\toprule
Method           &    Extra Stage     &       AP      &    AP$_{50}$  &    AP$_{75}$  \\
\midrule
GroupRevision-7B &     \checkmark    &      39.2     &      57.9     &      40.7     \\
MCR-GRPO (Ours)  & $\mathrm{\times}$ & \textbf{39.7} & \textbf{59.2} & \textbf{40.8} \\
\bottomrule
\end{tabular}
\begin{tabular}{lcccc}
\toprule
Method           &    Extra Stage    &     AR$_1$    &   AR$_{10}$   &   AR$_{100}$  \\
\midrule
GroupRevision-7B &     \checkmark    & \textbf{35.5} &      49.5     &      50.1     \\
MCR-GRPO (Ours)  & $\mathrm{\times}$ &      35.0     & \textbf{49.7} & \textbf{50.2} \\
\bottomrule
\end{tabular}
\caption{\textbf{Detailed COCO Comparison Between GroupRevision-7B and MCR-GRPO.} MCR-GRPO uses a single-stage inference protocol and achieves stronger overall detection performance. The best results are highlighted in \textbf{bold}.}
\label{tab:app-coco-stage}
\end{table}

These results suggest that the gain of MCR-GRPO does not rely on test-time revision. Instead, MCR-GRPO improves the model during training by assigning marginal contribution rewards to box-level object records, allowing the final model to produce stronger detections in a single inference pass.

\section{Qualitative Comparison}
We qualitatively compare our MCR-GRPO with VisionReasoner~\citep{visionreasoner} across Described Object Detection (DOD), Segmentation, Referring Expression Comprehension (REC), and Counting.
Figures~\ref{fig_dod}--\ref{fig_count} show the query, the prediction from each model, and the ground truth under the same task-specific visualization protocol.
These examples complement the aggregate results in the main paper by illustrating differences in instance selection, localization, and cardinality.

\paragraph{Described object detection.}
Figure~\ref{fig_dod} presents examples from the $\mathrm{D^3}$ benchmark~\citep{d3}, whose queries contain attributes, relations, and negated descriptions. In the displayed cases, MCR-GRPO more accurately identifies the complete set of objects satisfying the description while excluding visually similar distractors. This distinction is especially visible for multi-object queries, which require both description-consistent localization and correct cardinality.

\paragraph{Segmentation and REC.}
Figure~\ref{fig_seg} shows reasoning segmentation examples from ReasonSeg~\citep{reasonseg} and referring segmentation examples from RefCOCO, RefCOCO+, and RefCOCOg~\citep{refcoco}. All samples can also be used as REC examples. These queries require spatial-relation reasoning, ordinal disambiguation, or recognition from an indirect description. In the displayed cases, MCR-GRPO more reliably selects the intended instance and produces spatial outputs that align more closely with the target extent. For segmentation, both methods use the same SAM2-based mask-generation pipeline described in the main paper.

\paragraph{Counting.}
Figure~\ref{fig_count} compares the two models on PixMo-Count~\citep{pixmocount} and CountBench~\citep{countbench}. As in the main experiments, the predicted count is obtained from the number of generated object records, without additional counting-specific supervision. In the displayed scenes, MCR-GRPO more closely matches the ground-truth cardinality by reducing missed targets and redundant predictions while retaining a localized box for each counted instance.

\begin{figure*}[!t]
\centering
\includegraphics[width=0.8\textwidth]{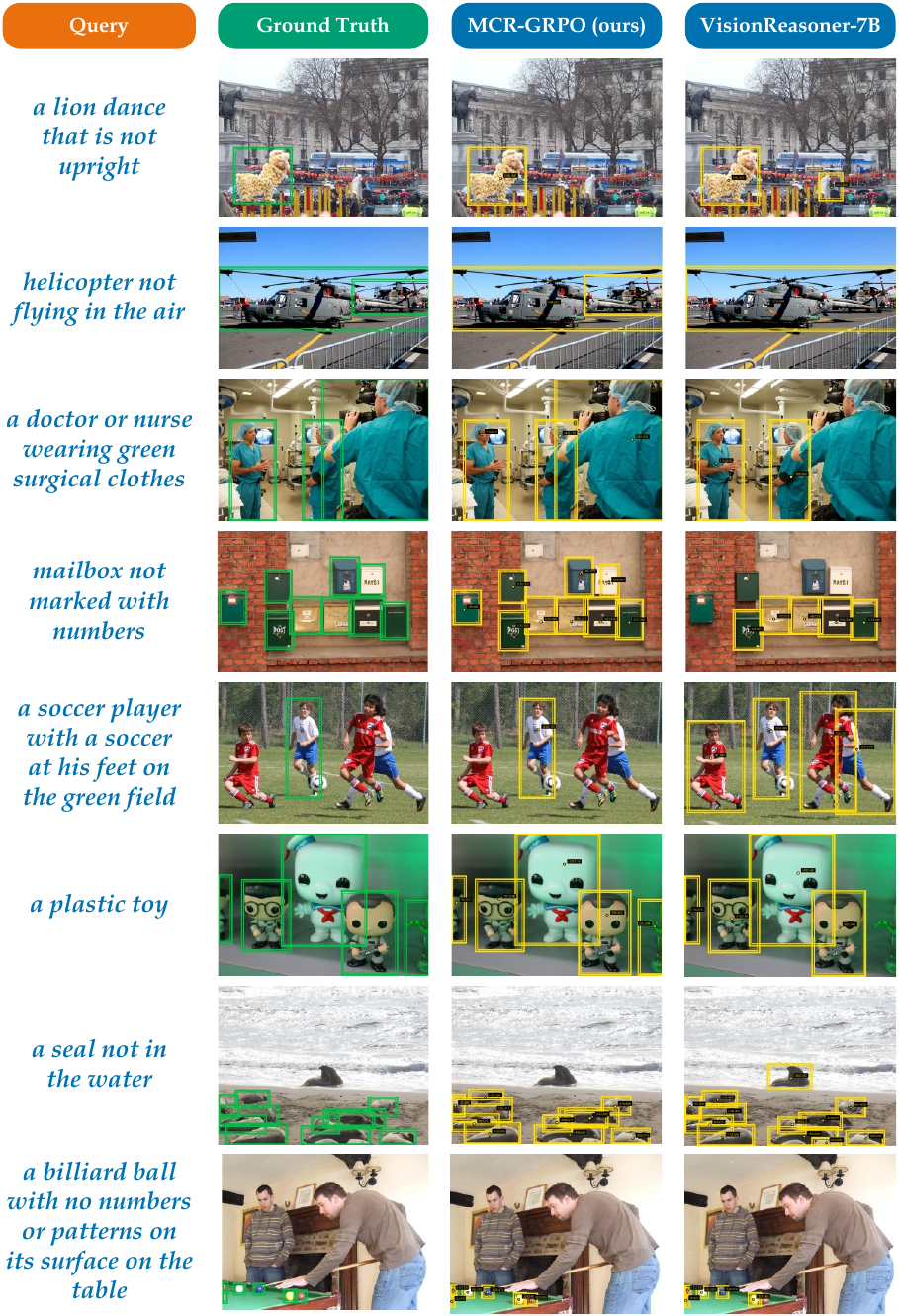}
\caption{\textbf{Qualitative Comparison with VisionReasoner on Described Object Detection.} Examples are drawn from the $\mathrm{D^3}$ benchmark. The displayed queries cover attribute-rich, relational, and negated descriptions. We present visualizations of the model-predicted boxes and points overlaid on the images.}\label{fig_dod}
\end{figure*}

\begin{figure*}[!t]
\centering
\includegraphics[width=\textwidth]{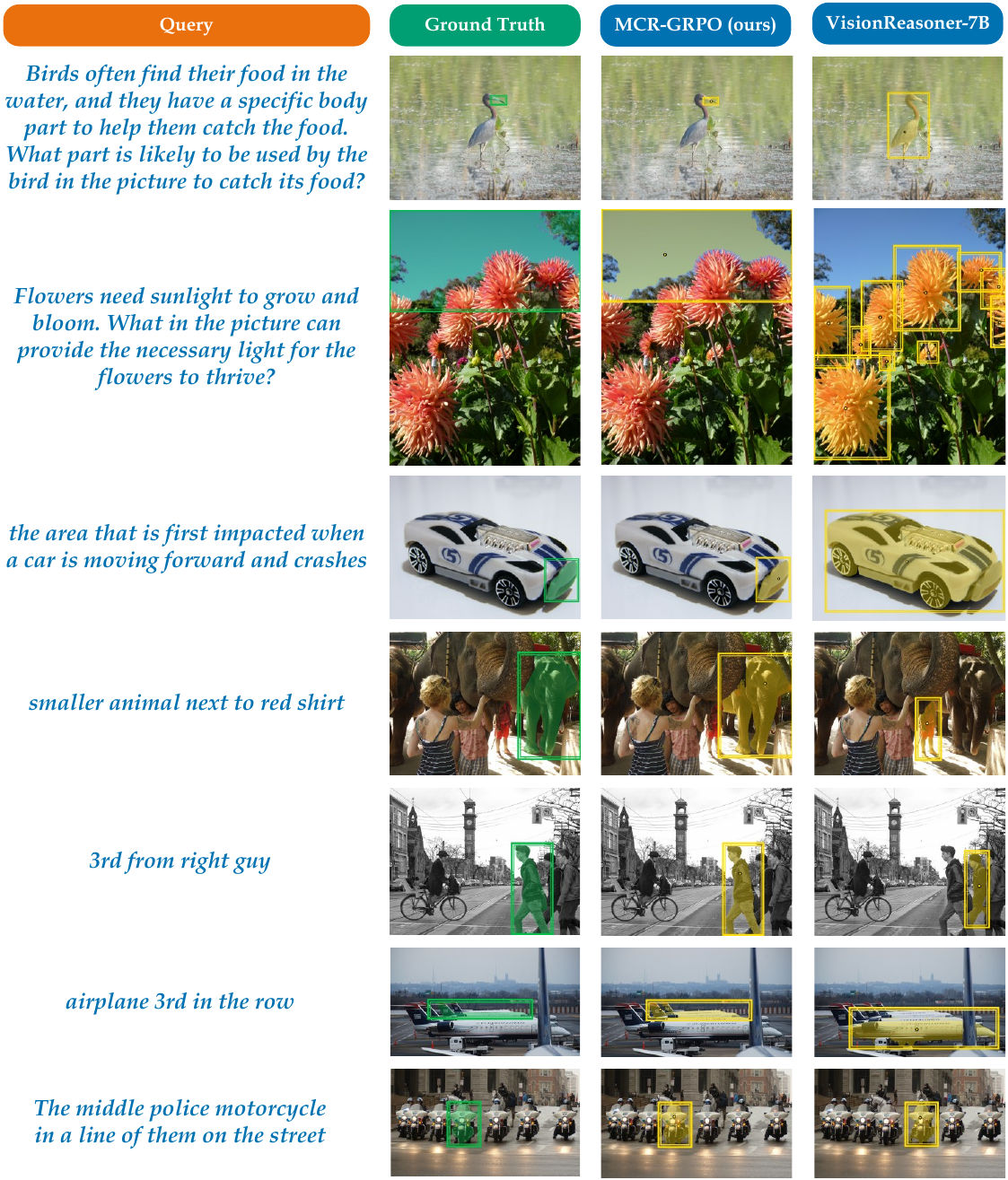}
\caption{\textbf{Qualitative comparison with VisionReasoner on Segmentation and Referring Expression Comprehension.} Examples are drawn from ReasonSeg and RefCOCO(+/g). The displayed queries involve reasoning-dependent descriptions, spatial relations, and ordinal cues. We present visualizations of the model-predicted boxes and points, as well as the masks generated by SAM2, overlaid on the images.}\label{fig_seg}
\end{figure*}

\begin{figure*}[!t]
\centering
\includegraphics[width=0.82\textwidth]{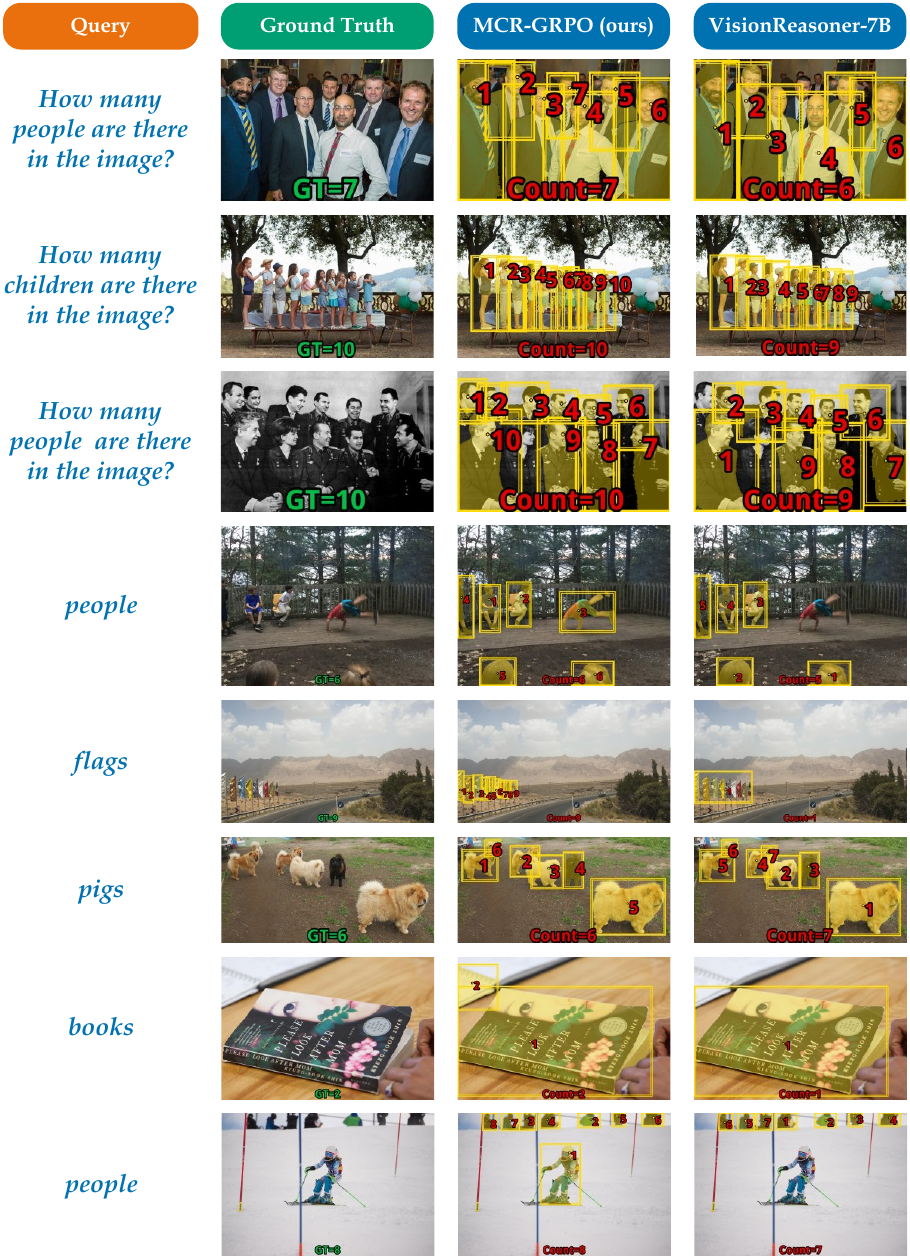}
\caption{\textbf{Qualitative Comparison with VisionReasoner on Counting.} Examples are drawn from PixMo-Count and CountBench. The predicted count is the number of localized object records. We present visualizations of the model-predicted boxes and points, as well as the masks generated by SAM2, overlaid on the images. In addition, we visualize the model's counting process and the resulting counts on the images.}\label{fig_count}
\end{figure*}

\section{Formal Analysis of Marginal Contributions}

In this section, we provide a formal analysis of the leave-one-out MCR credit used in the main text. The analysis separates the raw value difference $\Delta V_i$ from the normalized residual credit $\bar A_i^{\mathrm{mcr}}$. The raw value captures how removing one predicted record changes the count-normalized matched set value, while the subsequent within-response normalization converts these value differences into signed object-record credits for policy optimization. This analysis explains why unmatched records receive negative raw credit under over-prediction, and why redundant or weak records can be suppressed by negative normalized credit within the same response.

\paragraph{Setting.}
Let $S=\{d_i\}_{i=1}^{K}$ be the predicted object records and $G=\{g_j\}_{j=1}^{N}$ be the ground-truth objects. Pair scores satisfy $s(d_i,g_j)\ge 0$. For any predicted subset $T\subseteq S$, let $\mathcal M(T,G)$ denote the set of valid one-to-one matchings between $T$ and $G$. Following the matching objective in the main text, define:
\begin{equation}
M_T^\star = \arg\max_{M'\in\mathcal M(T,G)} \sum_{(r,j)\in M'}s(d_r,g_j).
\end{equation}
The matched set value for $T$ is:
\begin{equation}
V(T,G) = \frac{1}{\max(|T|,N)} \sum_{(r,j)\in M_T^\star}s(d_r,g_j).
\end{equation}
In particular, the raw leave-one-out value used by MCR is:
\begin{equation}
\Delta V_i = V(S,G)-V(S\setminus\{d_i\},G).
\end{equation}

For compactness, we write $M_S^\star$ for the optimal matching on $S$, and $M_{-i}^\star$ for the optimal matching on $S\setminus\{d_i\}$. A predicted record $d_i$ is called unmatched if there exists an optimal matching $M_S^\star$ that does not include $d_i$.

\begin{lemma}[Removal bounds]
\label{lem:removal}
For any $d_i\in S$: (a) Removing $d_i$ cannot increase the optimal matched sum:
\begin{equation}
\sum_{(r,j)\in M_{-i}^\star}s(d_r,g_j) \le \sum_{(r,j)\in M_S^\star}s(d_r,g_j).
\end{equation}

(b) If $d_i$ is matched to $g_{j(i)}$ in an optimal matching $M_S^\star$, then:
\begin{equation}
\sum_{(r,j)\in M_{-i}^\star}s(d_r,g_j) \ge \left[ \sum_{(r,j)\in M_S^\star}s(d_r,g_j)\right] - s(d_i,g_{j(i)}).
\end{equation}

(c) If $d_i$ is unmatched in some optimal matching $M_S^\star$, then:
\begin{equation}
\sum_{(r,j)\in M_{-i}^\star}s(d_r,g_j) = \sum_{(r,j)\in M_S^\star}s(d_r,g_j).
\end{equation}
\end{lemma}

\begin{proof}
(a) Any matching after removing $d_i$ remains feasible for the original predicted set by leaving $d_i$ unused. Since all pair scores are non-negative, the maximum matched sum cannot increase after removal. (b) If $d_i$ is matched, deleting the pair $(i,j(i))$ from $M_S^\star$ gives a feasible matching after removing $d_i$ with the stated weight.
(c) If $d_i$ is unmatched in an optimal matching, the same matching remains feasible after removal, and the bound (a) gives equality.
\end{proof}

\begin{proposition}[Regime-dependent raw sign]
\label{prop:raw-sign}
(a) If $K\le N$, then $\Delta V_i\ge 0$ for every $d_i\in S$. (b) If $K>N$ and $d_i$ is unmatched in some optimal matching, then:
\begin{equation}
\Delta V_i
=
-\frac{1}{K(K-1)}
\sum_{(r,j)\in M_S^\star}s(d_r,g_j)
\le 0.
\end{equation}

The inequality is strict whenever the matched sum is positive.
\end{proposition}

\begin{proof}
(a) When $K\le N$, both values use denominator $N$, so:
\begin{equation}
\Delta V_i = \frac{1}{N} \left( \sum_{(r,j)\in M_S^\star}s(d_r,g_j) - \sum_{(r,j)\in M_{-i}^\star}s(d_r,g_j) \right) \ge 0
\end{equation}
by Lemma~\ref{lem:removal}. (b) When $K>N$ and $d_i$ is unmatched, the matched sum is unchanged after removal, while the denominator changes from $K$ to $K-1$. Hence:
\begin{equation}
\begin{aligned}
\Delta {V_i}&=\frac{1}{K} \sum_{(r,j)\in M_S^\star}s(d_r,g_j)-\frac{1}{K-1}\sum_{(r,j)\in M_S^\star}s(d_r,g_j) \\
&=-\frac{1}{K(K-1)} \sum_{(r,j)\in M_S^\star}s(d_r,g_j),
\end{aligned}
\end{equation}
with strict negativity whenever the matched sum is positive.

Finally, within-response normalization preserves the ordering of raw marginal contributions. When the within-response standard deviation is nonzero, $\bar A_i^{\mathrm{mcr}}$ is a strictly increasing affine transformation of $\Delta V_i$. Therefore, for any two records $d_i$ and $d_j$,
\begin{equation}
\Delta V_i > \Delta V_j \quad \Rightarrow \quad \bar A_i^{\mathrm{mcr}} > \bar A_j^{\mathrm{mcr}} .
\end{equation}
Thus, MCR guarantees response-relative ordering rather than an absolute helpful-or-harmful sign per box. If all records have identical raw marginal contributions, the variance is zero and we set all normalized MCR credits to zero.

\end{proof}

\begin{proposition}[\textbf{Relative Credit After Normalization}]
\label{prop:relative-credit}
For $K>1$ and $\sigma_\Delta>0$ in Eqs.~(8)--(9), normalized MCR credit is:
\begin{equation}
\bar A_i^{\mathrm{mcr}}
=
\frac{\Delta V_i-\mu_\Delta}{\sigma_\Delta+\epsilon}.
\end{equation}
Thus $\bar A_i^{\mathrm{mcr}}$ preserves the ordering of $\Delta V_i$, and its sign is determined by whether $\Delta V_i$ is above or below the within-response mean $\mu_\Delta$.
\end{proposition}

\begin{proof}
The denominator $\sigma_\Delta+\epsilon$ is positive when $\sigma_\Delta>0$ and $\epsilon\ge 0$. Therefore normalization is a strictly increasing affine map of $\Delta V_i$. Centering makes records below the response mean receive negative residual credit, even when their raw $\Delta V_i$ is non-negative. If $\sigma_\Delta=0$, MCR sets all normalized credits to zero.
\end{proof}

\begin{proposition}[\textbf{Lowest Credits Under Over-Prediction}]
\label{prop:rank}
Let $K>N$. For any record $d_k\in S$ and any unmatched record $d_i$,
\begin{equation}
\Delta V_k\ge \Delta V_i.
\end{equation}
Equality holds iff:
\begin{equation}
\sum_{(r,j)\in M_{-k}^\star}s(d_r,g_j) = \sum_{(r,j)\in M_S^\star}s(d_r,g_j).
\end{equation}
Thus unmatched records and perfectly substitutable records occupy the lowest
normalized MCR credits whenever $\sigma_\Delta>0$.
\end{proposition}

\begin{proof}
For the unmatched record $d_i$,
\begin{equation}
\Delta V_i
=
\frac{1}{K}
\sum_{(r,j)\in M_S^\star}s(d_r,g_j)
-
\frac{1}{K-1}
\sum_{(r,j)\in M_S^\star}s(d_r,g_j).
\end{equation}
For any record $d_k$,
\begin{equation}
\Delta V_k
=
\frac{1}{K}
\sum_{(r,j)\in M_S^\star}s(d_r,g_j)
-
\frac{1}{K-1}
\sum_{(r,j)\in M_{-k}^\star}s(d_r,g_j).
\end{equation}
Lemma~\ref{lem:removal} gives:
\begin{equation}
\sum_{(r,j)\in M_{-k}^\star}s(d_r,g_j)
\le
\sum_{(r,j)\in M_S^\star}s(d_r,g_j),
\end{equation}
so $\Delta V_k\ge \Delta V_i$. Since normalization is strictly increasing when $\sigma_\Delta>0$, the same ordering holds after normalization.
\end{proof}

\begin{proposition}[\textbf{Bounded Residual Credit}]
\label{prop:bound}
For $K\ge 2$ and $\epsilon=0$ in Eq.~(9),
\begin{equation}
|\bar A_i^{\mathrm{mcr}}| \le \sqrt{K-1}.
\end{equation}

With $\epsilon>0$, the magnitude can only decrease. Therefore, for object tokens
$t\in\tau_i$ with
\begin{equation}
A_{m,t}^{\mathrm{MCR\text{-}GRPO}} = A_m^{\mathrm{seq}} + \lambda_{\mathrm{mcr}} A_{t,i}^{\mathrm{mcr}},
\end{equation}
the composite token advantage has the same sign as $A_m^{\mathrm{seq}}$ whenever
\begin{equation}
\lambda_{\mathrm{mcr}}\sqrt{K-1} < |A_m^{\mathrm{seq}}|.
\end{equation}
\end{proposition}

\begin{proof}
For object tokens $t\in\tau_i$, Eq.~(10) gives $A_{t,i}^{\mathrm{mcr}}=\bar A_i^{\mathrm{mcr}}$. Let $z_i=\bar A_i^{\mathrm{mcr}}$ with $\epsilon=0$. By the population-standard-deviation normalization in Eq.~(8), we have $\sum_i z_i=0$ and $\sum_i z_i^2=K$. Hence:
\begin{equation}
z_i^2 = \left(\sum_{k\ne i}z_k\right)^2 \le (K-1)\sum_{k\ne i}z_k^2 = (K-1)(K-z_i^2),
\end{equation}
which gives $|z_i|\le\sqrt{K-1}$. If $\epsilon>0$, each normalized value is multiplied by $\sigma_\Delta/(\sigma_\Delta+\epsilon)\le 1$. The sign statement follows from
the triangle inequality.
\end{proof}

These results explain how MCR turns leave-one-out value changes into box-level training signals. The raw value difference $\Delta V_i$ measures how much a predicted record contributes to the count-normalized matched set value. Under over-prediction, removing an unmatched or redundant record can improve the set value, yielding negative raw credit. When $K\le N$, raw contributions are non-negative, but within-response centering can assign negative normalized residual credit to records whose marginal contributions fall below the response mean. Thus, weak, redundant, or poorly localized records can still be suppressed in the policy update, while records with stronger marginal contributions receive larger MCR credits and are encouraged. In this way, MCR realizes the positive-and-negative box-level credit assignment described in the main text while preserving the original response-level GRPO comparison.

\section{Additional Qualitative Examples}

Figure~\ref{fig_think} visualizes representative MCR-GRPO inference outputs across the task families evaluated in the main paper. The examples cover DOD on $\mathrm{D^3}$~\citep{d3}, segmentation and REC on ReasonSeg~\citep{reasonseg} and RefCOCO(+/g)~\citep{refcoco}, and counting on PixMo-Count~\citep{pixmocount} and CountBench~\citep{countbench}.

For each input, the figure displays an excerpt from the generated \texttt{<think>} content together with the final prediction. The reasoning identifies visual evidence used to distinguish the target from distractors or enumerate the queried category, while the final output expresses the decision as localized boxes, segmentation masks, or a count. These examples show how the same structured inference interface supports different task-specific outputs.

\begin{figure*}[!t]
\centering
\includegraphics[width=0.9\textwidth]{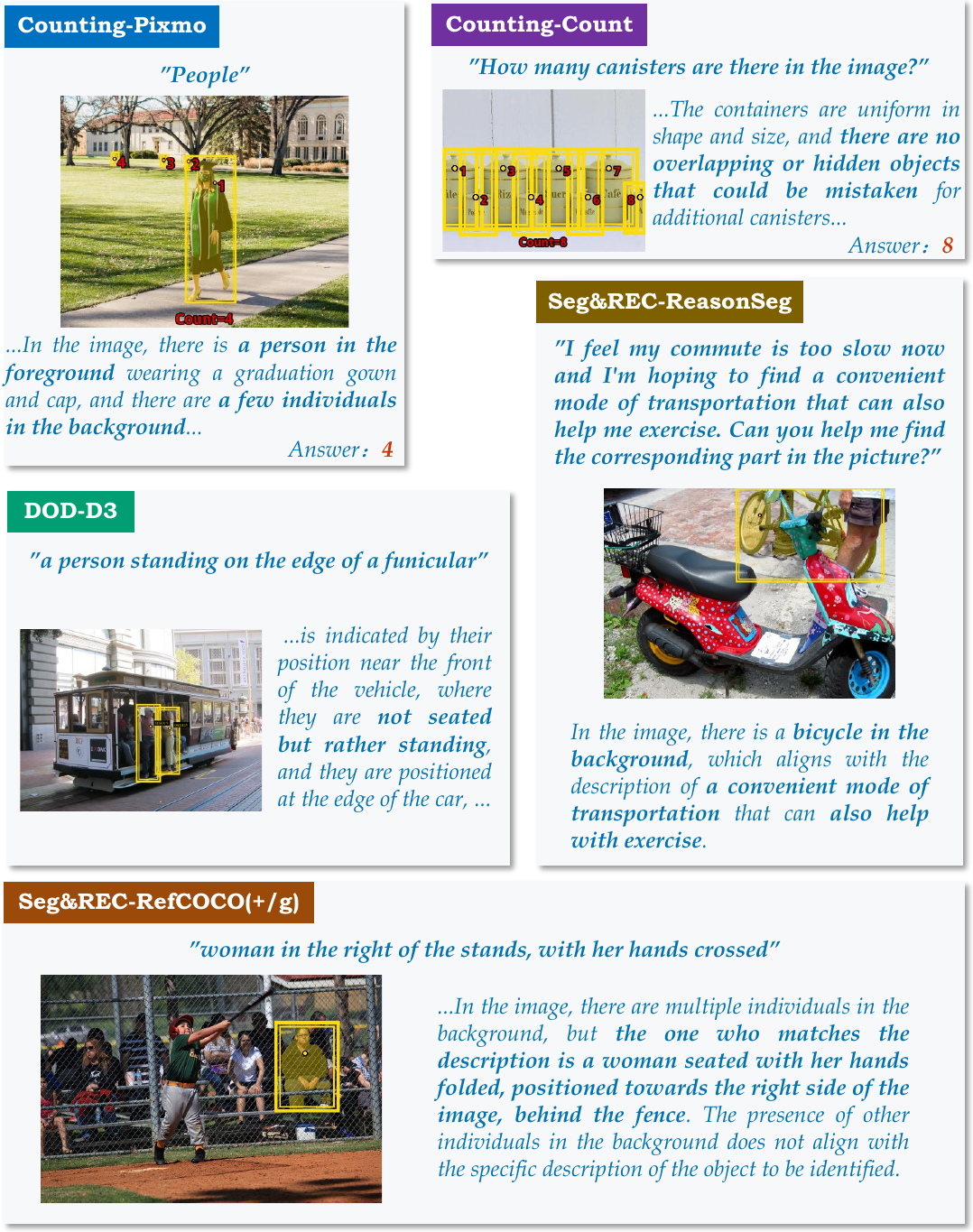}
\caption{\textbf{Cross-Task Reasoning and Prediction Examples from MCR-GRPO.} Representative examples cover DOD on $\mathrm{D^3}$, segmentation and REC on ReasonSeg and RefCOCO(+/g), and counting on PixMo-Count and CountBench. Each example pairs an excerpt from the generated \texttt{<think>} content with the final prediction, showing how attribute, relational, and cardinality reasoning is connected to task-appropriate boxes, masks, and counts.}\label{fig_think}
\end{figure*}

\renewcommand{\refname}{Supplementary References}
\putbib[ref]
\end{bibunit}

\end{document}